\documentclass[letterpaper]{article} 
\usepackage{aaai24}  
\usepackage{times}  
\usepackage{helvet}  
\usepackage{courier}  
\usepackage[hyphens]{url}  
\usepackage{graphicx} 
\usepackage{natbib}  
\usepackage{caption} 
\usepackage{algorithm}
\usepackage{algorithmic}
\usepackage{graphicx}
\usepackage{epsfig}
\usepackage{epstopdf}
\usepackage{stackengine,color}
\usepackage{amsmath,amssymb}
\usepackage{amsthm}
\newtheorem{theorem}{Theorem}
\usepackage{newfloat}
\usepackage{listings}
\DeclareCaptionStyle{ruled}{labelfont=normalfont,labelsep=colon,strut=off} 
\floatstyle{ruled}
\newfloat{listing}{tb}{lst}{}
\floatname{listing}{Listing}
\title{LISR: Learning Linear 3D Implicit Surface Representation Using Compactly Supported Radial Basis Functions}
\author{
	Atharva Pandey,
	Vishal Yadav,
	Rajendra Nagar,
	Santanu Chaudhury\\
}
\affiliations{
	
	Indian Institute of Technology Jodhpur, India\\ \{pandey.19, yadav.40, rn, santanuc\}@iitj.ac.in
	
}

\usepackage{bibentry}

\begin{document}
	
	\maketitle
	
	\begin{abstract}
		Implicit 3D surface reconstruction of an object from its partial and noisy 3D point cloud scan is the classical geometry processing and 3D computer vision problem.  In the literature,
		various 3D shape representations have been developed, differing in memory efficiency and shape retrieval effectiveness, such as volumetric, parametric, and implicit surfaces. Radial basis functions provide memory-efficient parameterization of
		the implicit surface. However, we show that training a neural network using the mean squared error between the ground-truth implicit surface and the linear basis-based implicit surfaces does not converge to the global solution. In this work, we propose locally supported compact radial basis functions for a linear representation of the implicit surface. This representation enables us to generate  3D shapes with arbitrary topologies at any resolution due to their continuous nature. We then propose a neural network architecture for learning the linear implicit shape representation of the 3D surface of an object. We learn linear implicit shapes within a supervised learning framework using ground truth Signed-Distance Field (SDF) data for guidance. The classical strategies face difficulties in finding linear implicit shapes from a given 3D point cloud due to numerical issues (requires solving inverse of a large matrix) in basis and query point selection. The proposed approach achieves better Chamfer distance and comparable F-score than the state-of-the-art approach on the benchmark dataset. We also show the effectiveness of the proposed approach by using it for the 3D shape completion task. 
	\end{abstract}
	
	\section{Introduction}
	3D surface reconstruction of real-world objects is a fundamental research problem in 3D Computer Vision, Geometry Processing, and Computer Vision with multiple applications such as Gaming, Animation, AR/VR, simulation, etc  \cite{berger2017survey} \cite{huang2022surface}. There exist various variants of this problem like 3D Reconstruction from single view images, multiple view images, single view/multiple view partial 3D scans, etc. \cite{berger2017survey,huang2022surface,xiu2022icon,tiwari2021neural}. The 3D Reconstruction problem can further be classified based on the representation of the 3D shape, differing in their memory efficiency and the effectiveness of shape retrieval. The representations can be broadly classified into four categories: point clouds, mesh-based representations, voxel-based representations, and implicit shape representations \cite{berger2017survey}. \\
	Implicit shape representation’s primary superiority lies in its ability to generate arbitrary shapes and topologies at any resolution due to their continuous nature \cite{macedo2011hermite}. They are modeled using parameterized functions, and learning the shape involves predicting these parameters \cite{park2019deepsdf}. Parameterization of the implicit representation defines the efficiency of the implicit surfaces. There exist works in the literature that aim to parameterize implicit surfaces \cite{park2019deepsdf,macedo2011hermite,yavartanoo20213dias}. One category of works aims at modeling the implicit representation by a function parameterized by the neural networks known as neural implicit \cite{park2019deepsdf,saito2019pifu,michalkiewicz2019implicit,xu2022hrbf}. Yavartanoo \emph{et al.} represents the implicit surface by a finite low-degree algebraic polynomial and learns the co-efficient of the algebraic polynomial \cite{yavartanoo20213dias}. This suffers from the fact that the low-degree polynomials fail to model the high curvature region of the surface. The radial basis functions have been used classically to model complex implicit surfaces \cite{macedo2011hermite,liu2016closed,xu2022hrbf}. 
	
	In this work, we propose a radial basis function-based representation \cite{liu2016closed,macedo2011hermite} to model implicit representation and learn the parameter of this linear representation using a neural network. The proposed Linear implicit shape model has significantly less number of parameters when compared to neural implicit \cite{park2019deepsdf} and algebraic implicit representation \cite{yavartanoo20213dias} and is a great choice for a task requiring memory efficiency due to less number of parameters. We learn linear implicit surface representation (LISR) from a partial point cloud using a based learning approach which is supervised by the ground truth Signed-Distance Field (SDF) of the underlying shape. In the methodology section, we show that directly using the RBF framework for training a neural network to predict the SDF does not work. For the convergence to the optimal solution,  We identify a constraint on the choice of basis and the query points for SDF evaluation to find the loss, which facilitates the learning of linear implicit surfaces. We show that the compactly-supported RBF as basis functions, along with a strategical choice of query points at the time of training, satisfy the desired constraint, thereby converging to the optimal solution. We run an experiment to show the validity of the constraint and train a neural network with the given constraint on the task of shape prediction and completion from a raw point cloud. We use the proposed approach for the shape completion task to show the application of the proposed approach. 
	
	In summary, our contributions are as follows:
	\begin{itemize}
		\item We parameterize the implicit 3D surface of an object using linear and compactly supported radial basis functions and identify a constraint on the choice of basis and query points for enabling efficient learning of linear implicit shapes.
		\item We show that the proposed learning-based Compactly Supported-RBF (CSRBF) representation along with the proposed query point selection strategy satisfies the constraint on the choice of the basis functions.
		\item We propose a neural network to learn the coefficient of linearity for CSRBF-based implicit surface representation. 
		\item We experimentally validate that the proposed approach can learn high-quality shape representations over distribution of shapes and show its applicability for the shape completion task.
	\end{itemize}
	The code for this paper is available at \url{https://github.com/Atharvap14/LISR}.
	\section{Related Works}
	\textbf{Volumetric Representation}: Volumetric representation models the surface surface using voxels in a unit cube. The learning-based approaches learn voxel occupancy from the given input point cloud/image of an object \cite{wu20153d,wu2016learning,peng2020convolutional,zhang20223dilg,wang2023alto,chibane2020implicit}. The accuracy of the reconstructed surfaces depends on the voxel size which is a cubic memory requirements in terms of the levels in the tree-based representation \cite{dai2017shape,riegler2017octnet}. However, volumetric representation-based methods still suffer in terms of huge memory requirements to represent high curvature regions of the surface. 
	
	\textbf{Neural Implicit Representation}: Recently, many algorithms were proposed to predict the signed distance field and the occupancy from a single point cloud using a neural network \cite{park2019deepsdf,erler2020points2surf,michalkiewicz2019deep,michalkiewicz2019implicit,mescheder2019occupancy,niemeyer2020differentiable,chabra2020deep,liu2021deep,venkatesh2021deep,chen2019learning}. These approaches use the continuous representation of the implicit surface, which they parameterize using the neural networks as compared to the volumetric representation, where the surface is represented by voxels by discretizing the 3D unit cube. There exist approaches that reconstruct the implicit surface from the single image of the object \cite{liu2019learning,sitzmann2019scene,yavartanoo20213dias,yu2022monosdf,saito2019pifu,xu2019disn}. Recently, many methods have decomposed the 3D surfaces into multiple primitives and learned implicit representation for each of the primitives from the local point cloud patches \cite{genova2019learning,genova2020local,deng2020nasa,jiang2020local,darmon2022improving,wu2022object,deng2020cvxnet,nan2017polyfit}.   These approaches do not solve the point completion problem and rely on a clean and complete point cloud of the object, and may not generalize well to unseen categories as they are trained in the global sense.    
	
	\textbf{Parametric Implicit Surfaces }: Another categories of algorithms parameterize the implicit surface representation and find optimal parameters either using classical optimization approaches \cite{huang2019variational,blane20003l,rouhani2010relaxing,macedo2011hermite,zhao2021progressive} or using learning-based approaches \cite{yavartanoo20213dias,xu2022hrbf}. Yavartanoo \emph{et al.} represents the implicit surface by a finite low-degree algebraic polynomial and learns the co-efficient of the algebraic polynomial \cite{yavartanoo20213dias}. This approach suffers from the fact that the low-degree polynomials fail to model the high curvature region of the surface. The radial basis functions have been used classically to model complex implicit surfaces \cite{macedo2011hermite,liu2016closed,xu2022hrbf} however have not been utilized for surface reconstruction in a learning-based framework. In this work, we propose an RBF-based representation to model implicit representation and learn the parameters of this linear representation using a neural network. The proposed Linear implicit shape models have significantly less number of parameters when compared to neural implicit \cite{park2019deepsdf} and algebraic implicit representation \cite{yavartanoo20213dias}.
	
	\textbf{Implicit Surfaces and Shape Completion}: 
	Recently, learning-based approaches have been proposed to predict the full 3D object
	from its partial 3D point cloud scan. The 3D surface
	completion algorithms differ based on the representation of the 3D objects, e.g.   triangle meshes, volumetric, 3D point clouds, and 3D implicit surface 
	functions. The approaches \cite{choy20163d,dai2017scannet,dai2017shape,hane2017hierarchical,sun2022patchrd} predict the completed 3D object in terms of its volumetric representation. However, These algorithms are computationally expensive for high-resolution volumetric point cloud completion. The triangle mesh-based 3D shape completion methods fail in capturing complex geometric structures \cite{litany2018deformable}. Implicit surfaces are well known for their interpolation property even in case of incomplete point clouds \cite{keren1994describing,turk2002modelling}. We use this property to solve the shape completion problem using implicit surface reconstruction. There exist various works \cite{yuan2018pcn,yan2022fbnet,yu2021pointr,xiang2022snowflake,mittal2022autosdf,yew2022regtr,yan2022shapeformer} for shape completion task, but there is no work to solve these two problems simultaneously. The algorithm proposed in \cite{chibane2020implicit} addresses this problem but mostly for human 3D shapes and also uses an occupancy network to first learn the implicit representation and subsequently solve the completion problem.

	\section{Methodology}
	In this work, we aim to solve the problem of implicit 3D surface reconstruction of an object $\mathcal{S}$ from a scene from its partial and noisy 3D point cloud scan. Let $\mathcal{M}=\{\mathbf{x}_1,\ldots,\mathbf{x}_n\}\subset\mathcal{S}$ be a 3D point cloud scan of an object. Given $\mathcal{M}$, our goal is to find an implicit representation of the 3D surface of the underlying object represented by $\mathcal{M}$. We represent the implicit surface of an object by the signed distance field $f:\mathbb{R}^3\rightarrow \mathbb{R}$. The value $f(\mathbf{x})$ represents the distance of the point $\mathbf{x}\in\mathbb{R}^3$ from the surface of the object $\mathcal{M}$. The zero level set of the SDF $f$ represents the surface of the underlying object $\mathcal{M}$, i.e., $\mathcal{S}=\{\mathbf{x}\in\mathbb{R}^{3}\mid f(\mathbf{x})=0\}$. In this work, we propose a Linear Implicit Surface Representation (LISR) of $f$ which we discuss in the following Sections.
	\subsection{Linear Implicit Surface Representation}
	In the proposed Linear Implicit Shape Representation (LISR) framework, we represent the surface of a 3D object as a zero-level set of a scalar field modeled using a linear combination of basis functions. The LISR representation is defined by a set of coefficients and a set of basis functions which is defined in Equation \eqref{eq1}.
	\begin{equation}
		f(\mathbf{x}) = \sum_{i=1}^{n} \alpha_i \phi_i(\mathbf{x}) 
		\label{eq1}
	\end{equation}
	Here, $f(\mathbf{x})$ represents the implicit function value at point $\mathbf{x}$, $\alpha_i$'s are the coefficients of linear combinations, and the functions $\phi_i$  are the basis functions evaluated at $\mathbf{x}$. Given a point cloud $\mathcal{M}$ with its ground truth signed distance field $\mathbf{s}_{\text{gt}}(\mathbf{x})$, we formulate an optimization problem where we minimize the L2 loss between the predicted SDF and the ground truth SDF. Formally, we define the optimization problem in Equation \eqref{eq2}.
	\begin{equation}
		\arg\min_{\boldsymbol{\alpha}} \sum_{\mathbf{x} \in \mathcal{Q}} \| f(\mathbf{x}) - s_{\text{gt}}(\mathbf{x}) \|_2^2
		\label{eq2}
	\end{equation}
	Here, $\boldsymbol{\alpha}=\begin{bmatrix}\alpha_1&\alpha_2&\cdots&\alpha_n\end{bmatrix}^\top\in\mathbb{R}^n$ contains the coefficients to be optimized, $\mathcal{Q}$  represents the domain of the optimization problem, and $\|\cdot\|_2$ denotes the L2 norm. It is easy to observe that the optimization problem defined in Equation \eqref{eq2} can be rephrased as a solution to a linear system $\mathbf{V}^\top\boldsymbol{\alpha} = \mathbf{s_{\text{gt}}}$ using the condition $f(\mathbf{x})=\mathbf{s}_{gt}(\mathbf{x}),\;\forall \mathbf{x}\in\mathcal{Q}$. Here, $\mathbf{s_{\text{gt}}}$ is the vector containing ground truth SDF values for query points $ \{\mathbf{x}_1, \mathbf{x}_2, \ldots, \mathbf{x}_m\}$  in the domain $\mathcal{Q}$, $\mathbf{V}^\top\in\mathbb{R}^{m\times n}$ is the matrix of basis functions evaluated at query points ($n$ is the number of basis functions, and $m$ is the number of query points) and defined as  $\mathbf{V}^\top = \begin{bmatrix} \phi_1(\mathbf{x}_1) & \cdots & \phi_n(\mathbf{x}_1) \\ \vdots  & \ddots & \vdots \\\phi_1(\mathbf{x}_m) & \cdots & \phi_n(\mathbf{x}_m)\end{bmatrix}$.  We first determine the condition where we can determine the optimal solution to this linear system which we state in the below result. Let $\boldsymbol{\alpha}^*$ represent the solution obtained by solving the optimization problem defined in Equation \eqref{eq1} (We use the gradient descent algorithm). 
	\begin{theorem}
		\label{th_1}
		The solution $\boldsymbol{\alpha}^*$ will be equal to the optimal solution of the linear system $\mathbf{V}^\top\boldsymbol{\alpha} = \mathbf{s_{\text{gt}}}$ if and only if $\mathbf{VV}^\top$ is a full rank matrix.
	\end{theorem}
	\begin{proof}
		To prove this result, we first show that the gradient of the SDF loss function concerning $\boldsymbol{\alpha}$ is equal $\mathbf{VV}^\top(\alpha - (\mathbf{VV}^\top)^{-1}\mathbf{V}\mathbf{s}_{gt})$. Consider the the SDF Loss function  $\ell(\boldsymbol{\alpha}) = \sum_{\mathbf{x} \in \mathcal{Q}} \| f(\mathbf{x}) - \mathbf{s}_{\text{gt}}(\mathbf{x}) \|_2^2$ as defined in Equation \eqref{eq1}. Here, $f(\mathbf{x})=\mathbf{V}^\top(x)\boldsymbol{\alpha}$ and $\mathbf{V}(\mathbf{x})=\begin{bmatrix}\phi_1(\mathbf{x})&\phi_2(\mathbf{x})&\cdots&\phi_n(\mathbf{x})\end{bmatrix}^\top$ is the vector of values of basis functions evaluated at query point $\mathbf{x}\in\mathcal{Q}$. Now, using the chain rule, we can easily show that the gradient of the SDF loss with respect to $\boldsymbol{\alpha}$ can be defined as:
		\begin{eqnarray}
			\nabla_{\boldsymbol{\alpha}}\ell&=&2 \sum_{\mathbf{x} \in \mathcal{Q}} (\mathbf{V}^\top(x)\boldsymbol{\alpha} - \mathbf{s}_{\text{gt}}(\mathbf{x}))\nabla_{\boldsymbol{\alpha}} (\mathbf{V}^\top(x)\boldsymbol{\alpha})\\
			&=&2\mathbf{VV}^\top\left(\boldsymbol{\alpha} - (\mathbf{VV}^\top)^{-1}\mathbf{V}\mathbf{s}_{\text{gt}}\right)
		\end{eqnarray}
		We observe that $\boldsymbol{\alpha}^* = (\mathbf{VV}^\top)^{-1}\mathbf{V}\mathbf{s}_{\text{gt}}$ represents the optimal solution for the linear system $ \mathbf{V}^\top\alpha = \mathbf{s}_{\text{gt}} $.  Hence, the optimal solution of the gradient descent algorithm will be  $\boldsymbol{\alpha}_{0}+\mathbf{VV}^\top\boldsymbol{\gamma}^*$. Here, $\boldsymbol{\gamma}^*=\underset{\boldsymbol{\gamma}}{\arg\min}\|\boldsymbol{\alpha}_{0}+\mathbf{VV}^\top\boldsymbol{\gamma}-\boldsymbol{\alpha}^*\|^2_2$. Here, $\boldsymbol{\alpha}_0$ is the initialized values and $\boldsymbol{\alpha}_{0}+\mathbf{VV}^\top\boldsymbol{\gamma}^*$ represents that optimal $\boldsymbol{\alpha}$ lies in the subspace of $\mathbf{VV}^\top$ centered at $\boldsymbol{\alpha}_0$. We can easily deduce (using least square solution) that the optimal $\boldsymbol{\gamma}^*$ is the solution to a linear system $\mathbf{VV}^\top\boldsymbol{\gamma}= \left( \boldsymbol{\alpha}^*-\boldsymbol{\alpha}_{0}\right)$. Unique value of $\boldsymbol{\gamma}$ exists if and only if $\mathbf{VV}^\top$ is a full-rank matrix and hence the gradient descent algorithm will be able to converge if and only if $\mathbf{VV}^\top$ is a full-rank matrix.
	\end{proof}
	\begin{figure}[!h]
		\centering
		\stackunder{\epsfig{figure=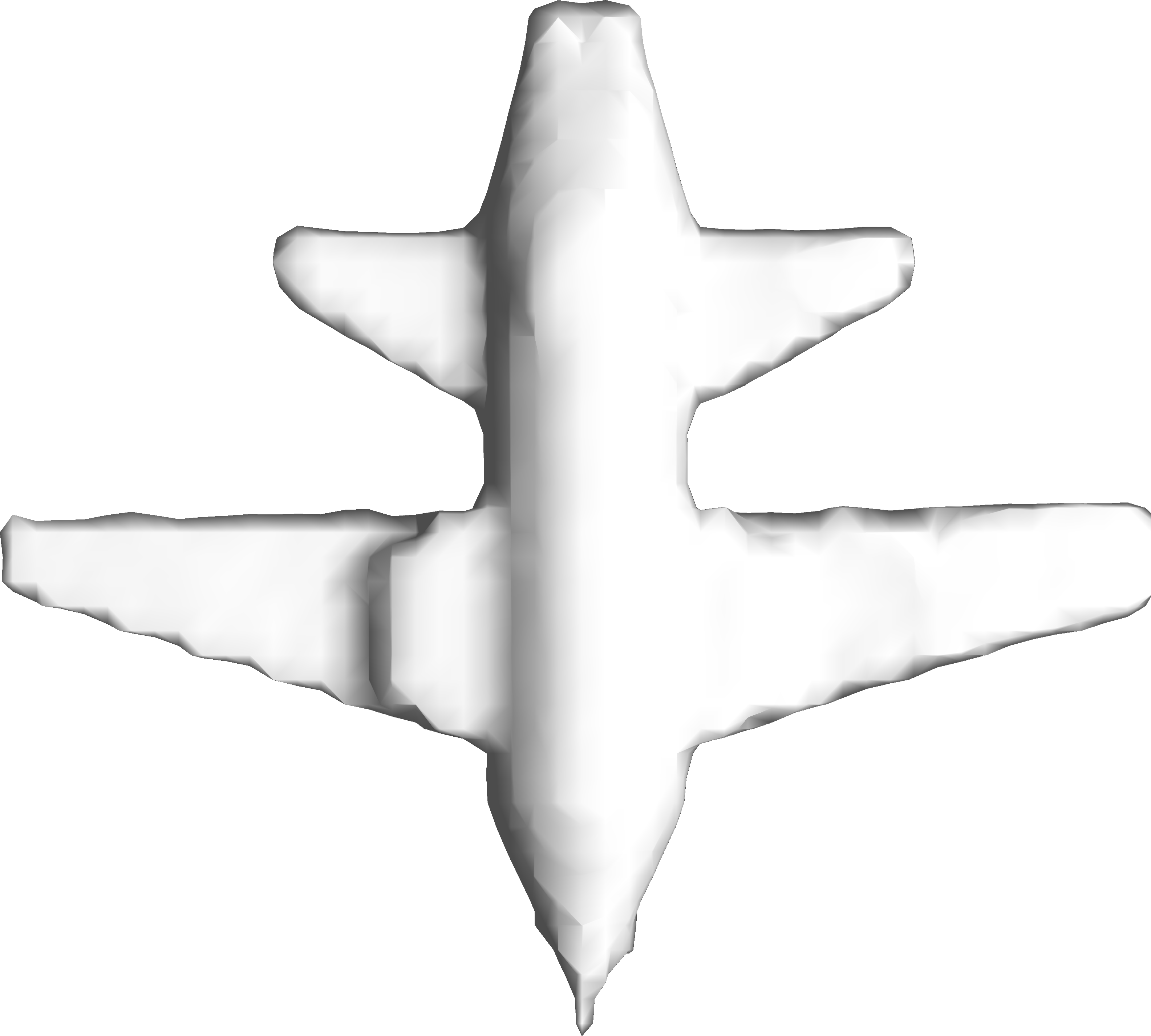,width=0.27\linewidth}}{$\eta=10^{-6}$}
		\stackunder{\epsfig{figure=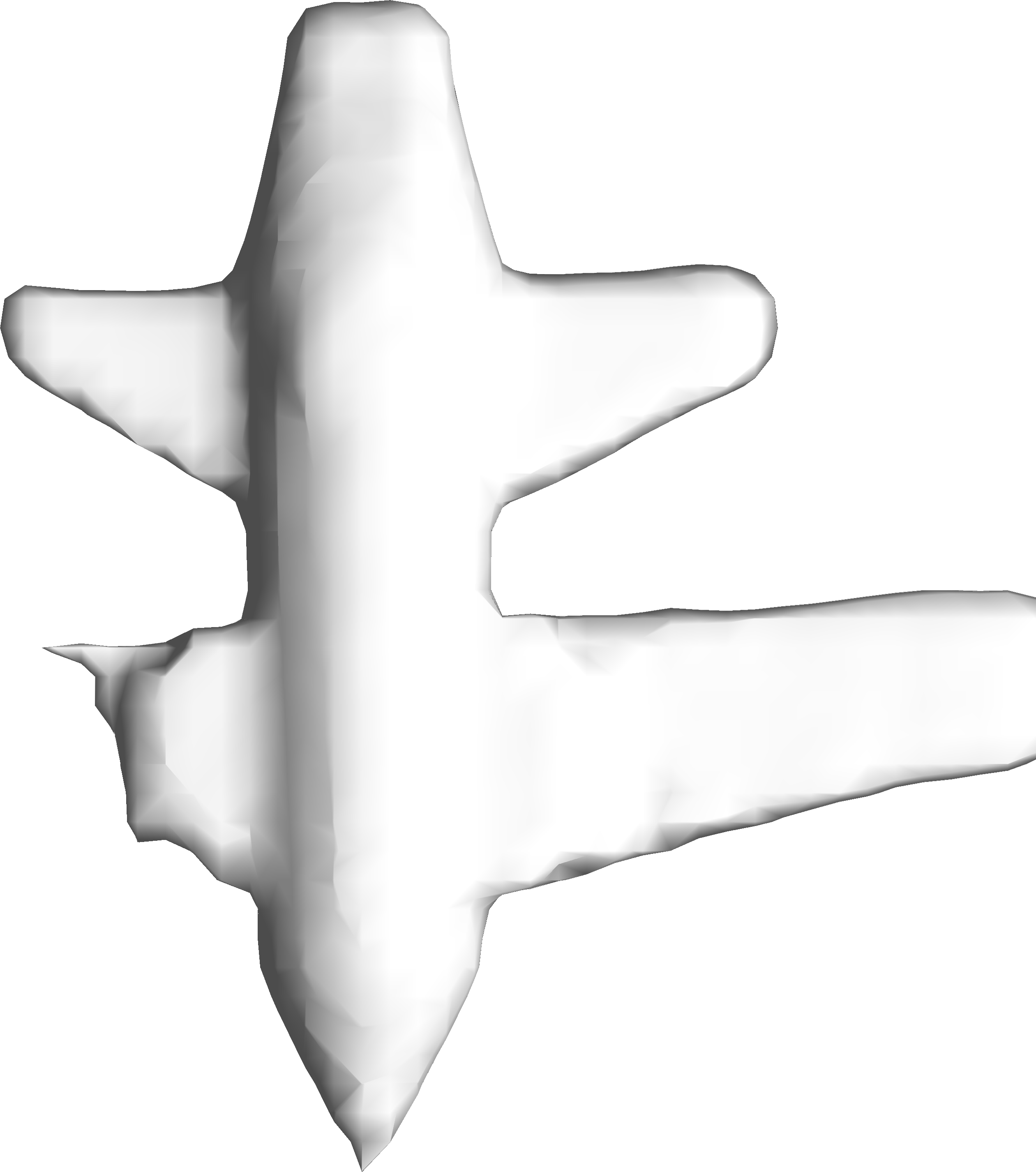,width=0.22\linewidth}}{$\eta=10^{-5}$}
		\stackunder{\epsfig{figure=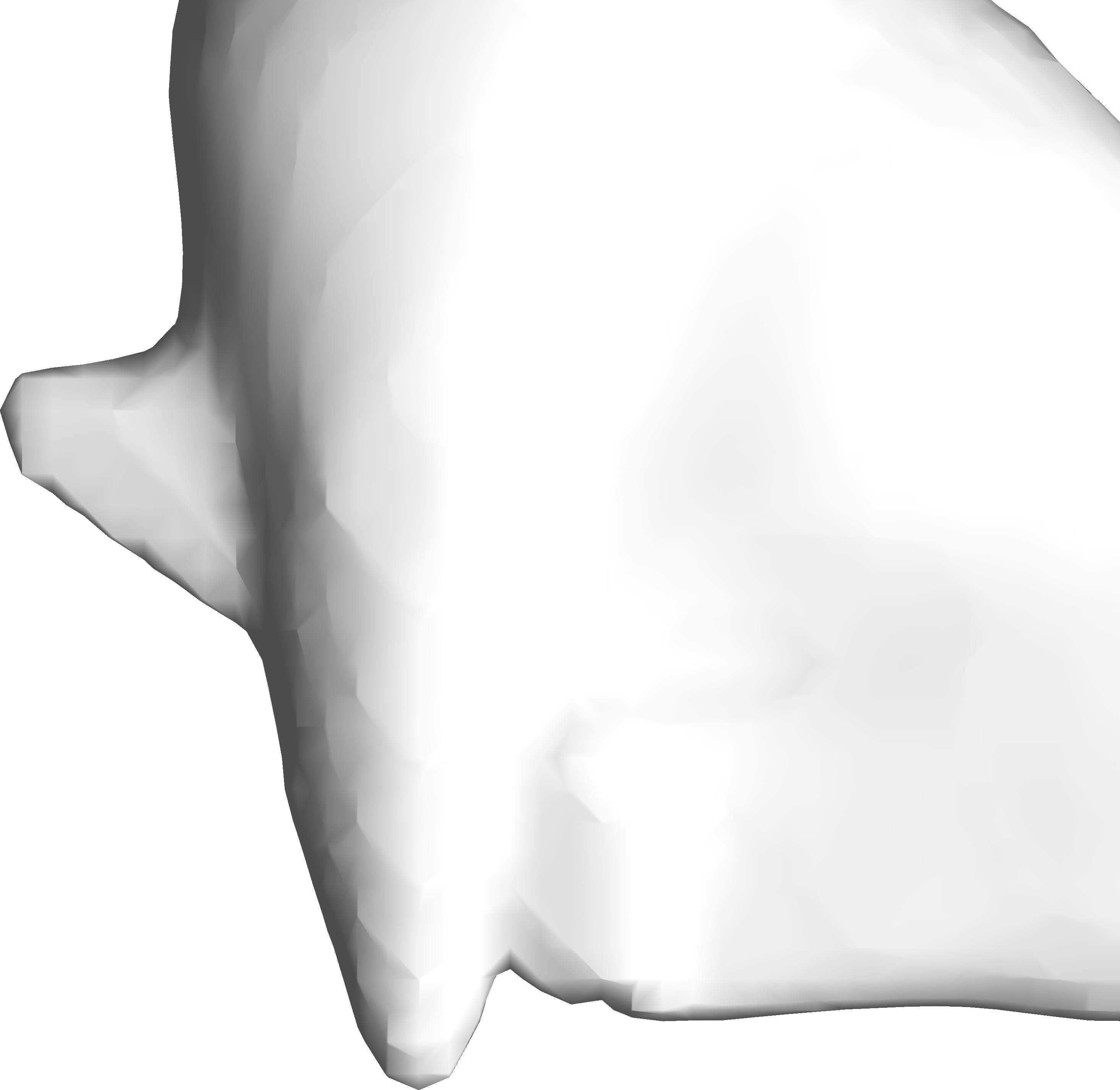,width=0.21\linewidth}}{$\eta=10^{-4}$}
		\stackunder{\epsfig{figure=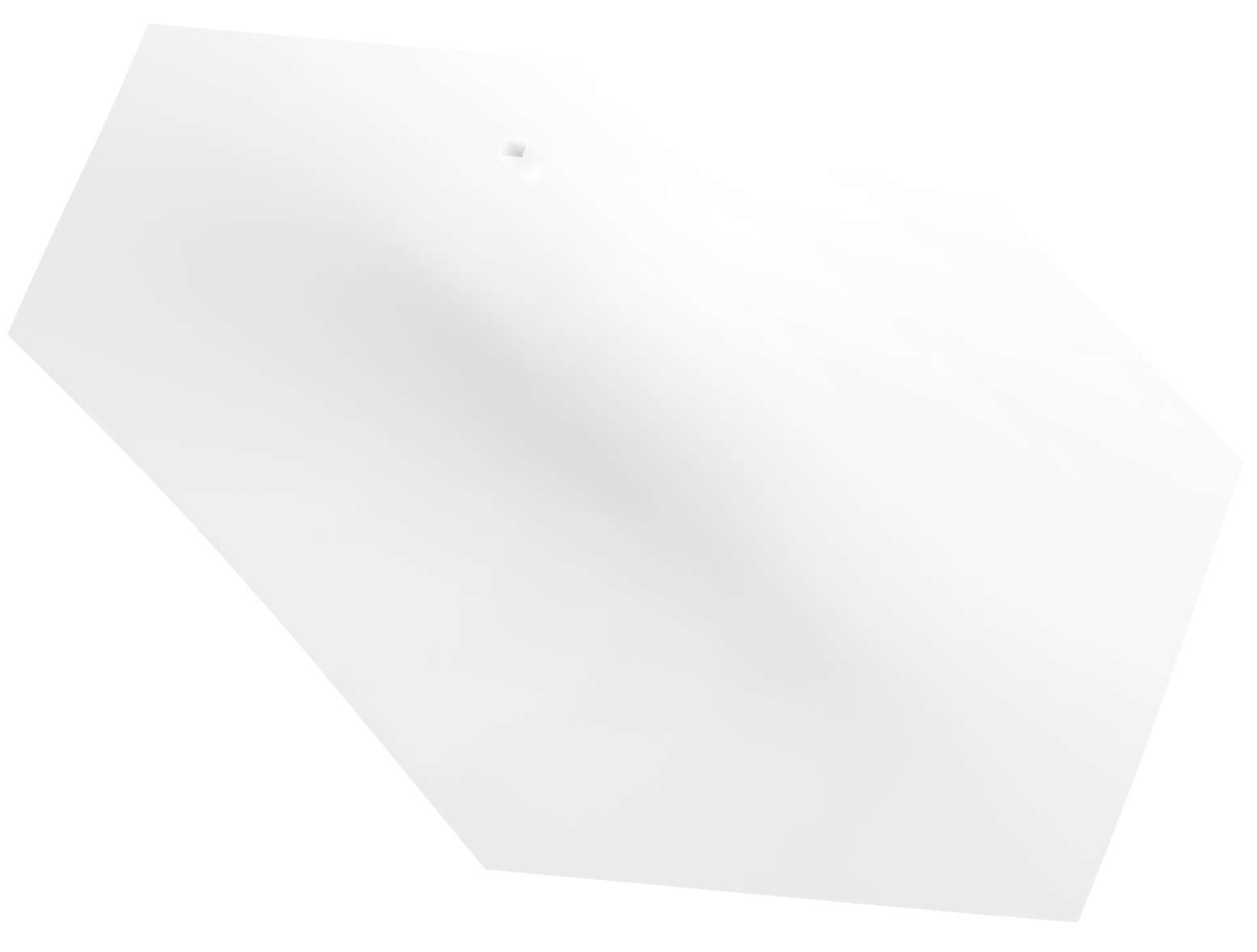,width=0.22\linewidth}}{$\eta=10^{-2}$}
		\caption{Perturbation of HRBF coefficients: Noise was added to each $\beta$ value and the noise was uniformly sampled from the range $[-\eta,\eta]$. We observe that the shape is too sensitive to $\beta$ as even for $\eta=10^{-5}$ one of the wings of the plane has been clipped.}
		\label{betaperturb}
	\end{figure}
	\subsection{Compactly Supported Radial Basis Functions}
	In the LISR, we experimentally observe that the basis functions $\phi_i$'s with global support, such as polynomials \cite{blane20003l,rouhani2010relaxing} and Hermite-RBFs \cite{macedo2011hermite}, do not meet the rank criteria defined in Theorem \ref{th_1} for all the query points $\mathbf{x}\in\mathcal{Q}$. The shapes generated using these basis functions are highly sensitive to coefficients, as shown in Figure \ref{betaperturb}. This makes it harder to learn the distribution of coefficients over a class of shapes. In contrast, the use of compactly supported basis functions reduced the sensitivity of the entire shape to coefficient perturbations and enhanced the correlation between coefficients and their respective basis functions. Intuitively, this implies that only a specific subset of basis functions and their corresponding coefficients control the shape representation in localized regions of the 3D space. Thereby reducing the inter-dependence of coefficients and increasing the matrix rank. Therefore, we propose to use the compactly supported radial basis functions as our basis functions $\phi_i$'s. We define the local support for each basis function using the Voronoi space partition of kernel points. We modify the LISR defined in Equation \eqref{eq1} and define the modified linear implicit surface representation in Equation \eqref{eq6}. 
	\begin{equation}
		f(\mathbf{x}) = \sum_{i=1}^{q} r_i(\mathbf{x}) \boldsymbol{\beta}_i^\top\mathbf{ \nabla} (\|\mathbf{x}-\mathbf{p}_i\|_2^3).
		\label{eq6}
	\end{equation}
	Here, $\boldsymbol{\beta}_i=\begin{bmatrix}\beta_{i,x}& \beta_{i,y}& \beta_{i,z}\end{bmatrix}^\top\in\mathbb{R}^3$ is a vector formed by three unknown parameters, $\mathbf{p}_i$ is the kernel point and $r_i(\mathbf{x})$ is the function defining the local support for basis with kernel point $\mathbf{p}_i$,i.e, if the point $\mathbf{x}$ lies in the local support of $\mathbf{p}_i$ then  $r_i(\mathbf{x}) $ evaluates to one, zero otherwise.
	\subsection{Query Point Selection}
	
	Along with a proper choice of basis functions, the selection of query points is a critical step to satisfy the full rank matrix criteria, as all the query points do not satisfy the rank criterion (Theorem \ref{th_1}). Our strategy for query point selection is to select at least three linearly independent query points from each local support. This strategy enables us to make the submatrix of $\mathbf{V}^\top$ corresponding to the particular local support a full rank matrix. 
	\begin{theorem}
		The matrix $\mathbf{VV}^\top$ is a full rank matrix for the linear implicit surface representation using the proposed compactly supported radial basis functions and the query point selection strategy.
	\end{theorem}
	\begin{proof}
		Consider the CSRBF formulation of linear implicit surface as defined in  Equation \eqref{eq6}. We rewrite the  SDF $f(\mathbf{x})$ by substituting the expanded gradients and coefficients as: 
		\begin{multline}
			f(\mathbf{x}) = \sum_{i=1}^{q} r_i(\mathbf{x}) \biggl( \beta_{i,x} \frac{\partial}{\partial x}(\| \mathbf{x} - \mathbf{p}_i \|_2^3) \\
			+ \beta_{i,y} \frac{\partial}{\partial y}(\| \mathbf{x} - \mathbf{p}_i \|_2^3) + \beta_{i,z} \frac{\partial}{\partial z}(\| \mathbf{x} - \mathbf{p}_i \|_2^3) \biggr)
		\end{multline}
		Now, the proposed CSRBF-based linear implicit surface representation can be written as
		\begin{eqnarray}
			f(\mathbf{x}) &=& \sum_{i=1}^{3q} \alpha_i \phi_i(\mathbf{x})
			\label{eq8}.
		\end{eqnarray}
		Where, $\phi_i$'s and $\alpha_i$'s  are defined as in Equations \eqref{eq9} and \eqref{eq10}, respectively. 
		\begin{equation}
			\phi_i(\mathbf{x}) = \begin{cases}
				r_{\frac{i+2}{3}}(\mathbf{x})\frac{\partial}{\partial x}(\| \mathbf{x} - \mathbf{p}_k \|_2^3), & \text{if } i = 3k+1 \\
				r_{\frac{i+1}{3}}(\mathbf{x})\frac{\partial}{\partial y}(\| \mathbf{x} - \mathbf{p}_k \|_2^3), & \text{if } i = 3k+2 \\
				r_{\frac{i}{3}}(\mathbf{x})\frac{\partial}{\partial z}(\| \mathbf{x} - \mathbf{p}_k \|_2^3), & \text{if } i = 3k \\
			\end{cases}
			\label{eq9}
		\end{equation}
		\begin{equation}
			\alpha_i(\mathbf{x}) = \begin{cases}
				\beta_{{\frac{i+2}{3}},x}, & \text{if } i = 3k+1 \\
				\beta_{{\frac{i+1}{3}},y} & \text{if } i = 3k+2 \\
				\beta_{{\frac{i}{3}},z}, & \text{if } i = 3k. \\
			\end{cases}
			\label{eq10}
		\end{equation}
		Now, with this modified definition of LISR, we can rewrite the $i$-th row of the matrix $\mathbf{V}^\top$ as: $\begin{bmatrix}\mathbf{0}_{m_i\times3i-3} & \mathbf{M}_{i} & \mathbf{0}_{m_i\times{3q-3i}}\end{bmatrix}$
			where $m_i$ represents the number of query points present in local support corresponding to $\mathbf{p}_i$ and $$
			\mathbf{M}_i = \begin{bmatrix}
				\phi_{3i-2}(x_{i,1}) & \phi_{3i-1}(x_{i,1}) & \phi_{3i}(x_{i,1}) \\
				\vdots &\vdots  & \vdots \\
				\phi_{3i-2}(x_{i,m_i}) & \phi_{3i-1}(x_{i,m_i}) & \phi_{3i}(x_{i,m_i}) \\
			\end{bmatrix}.$$
			\begin{figure*}[h!] 
				\centering
				\includegraphics[width=0.95\linewidth]{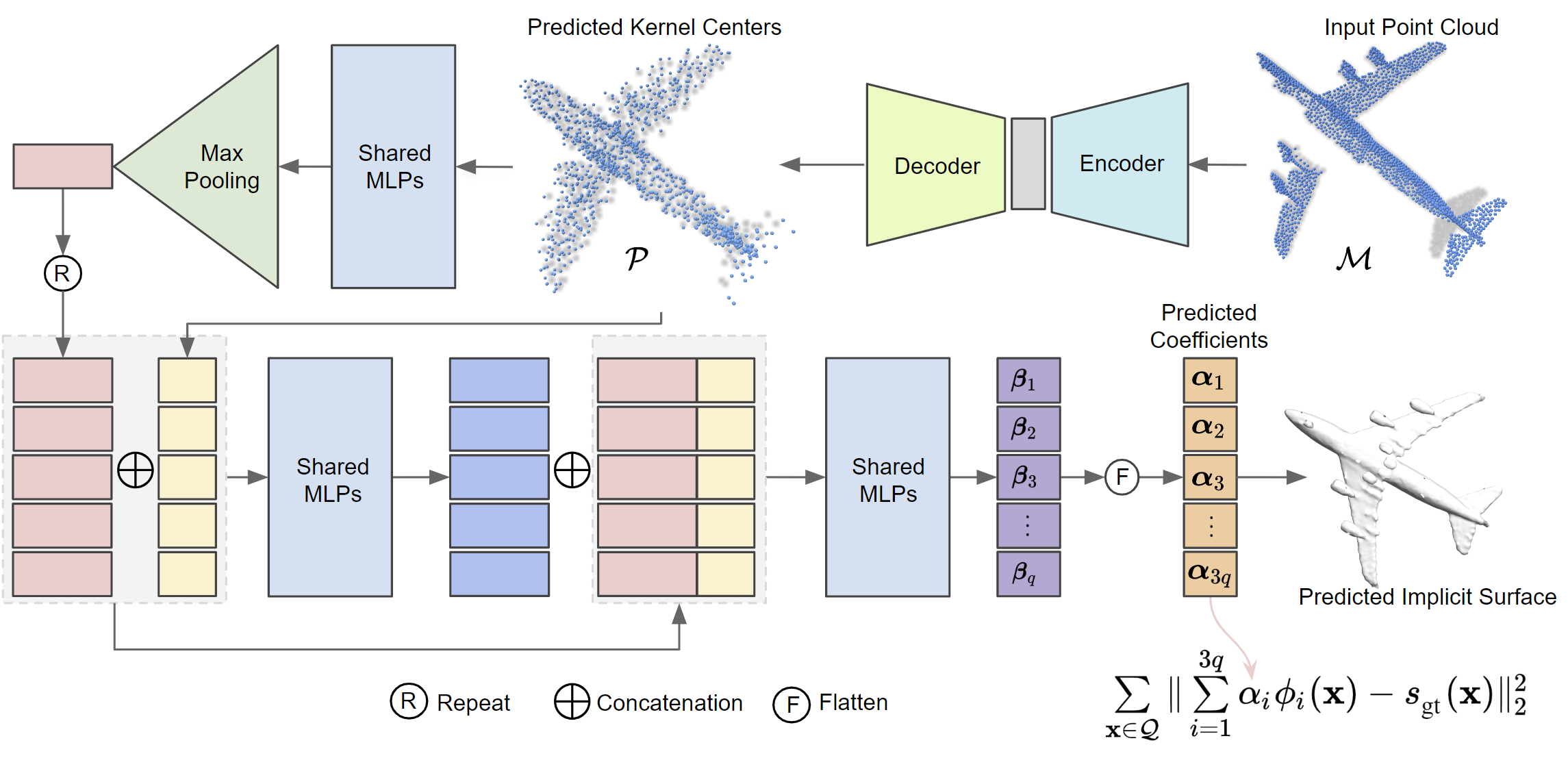}
				\caption{The proposed framework consists of two interconnected networks: the Basis Prediction Network (top) and the Coefficient Prediction Network (bottom). The Basis Prediction Network predicts the kernel points of the Compactly Supported Radial Basis Function (CSRBF) for partial shapes, while for complete shapes, the input point cloud can directly serve as the kernel points for basis. The learned basis kernel points are then passed into the Coefficient Prediction Network, which estimates the coefficients corresponding to each basis. Together, these predictions form a linear implicit shape representation for a given input point cloud.}
				\label{fig:fc}
			\end{figure*}
			For $\mathbf{V}^\top$ to be a full-rank matrix, each of the $\mathbf{M}_i$ has to be full-rank. It is trivial to observe that $\mathbf{M}_i$ is a full-rank as the query points are selected along the $m_i$ linearly independent directions (Algorithm \ref{alg:algorithm2}). Since $\mathbf{M}_i$ is a full-rank linear transformation on query points, the query point matrix must be full-rank (which is 3). Hence there must exist at least three linearly independent points in the set of query points belonging to each local support such that $\mathbf{VV}^\top$ is a full-rank matrix.
			Hence, the proposed query point sampling strategy is sufficient for learning the linear implicit shape with a CSRBF basis. 
		\end{proof}
		To improve the rate of convergence, we sample the points $\{ \mathbf{p}_i+\epsilon \hat{i},\mathbf{p}_i+\epsilon \hat{j},\mathbf{p}_i+\epsilon \hat{k}\}$ for local support of $\mathbf{p}_i$ instead of sampling any three linearly independent points in local support. The value of $\epsilon$ is common across local supports and chosen such that all sampled query points lie in their respective local supports. The gradient of the SDF loss can be written as follows:
		$$\nabla_{\boldsymbol{\alpha} } \left(\sum_{\mathbf{x}\in\mathcal{Q}} \| f(\mathbf{x}) - S_{\text{gt}}(\mathbf{x}) \|_2^2\right) = 2\mathbf{V}\mathbf{V}^\top\left(\boldsymbol{\alpha} - \boldsymbol{\alpha}^*\right)$$
		where $\boldsymbol{\alpha}^* = (\mathbf{V}\mathbf{V}^\top)^{-1}\mathbf{V}\mathbf{s}_{\text{gt}}$ is the solution to the optimization problem. The sampling strategy in Algorithm \ref{alg:algorithm2} ensures that $\mathbf{V}\mathbf{V}^\top = c\mathbf{I}$ where $c$ is an arbitrary constant and $\mathbf{I}$ is the identity matrix. This shows that gradient descent always occurs in the direction of global optima. This ensures that all dimensions are similarly scaled allowing the model to train faster.
		\begin{algorithm}[tb]
			\caption{Faster-Convergence Query Point Selection}
			\label{alg:algorithm2}
			\textbf{Input}: Basis Kernel Points $\{\mathbf{p}\}_{i=1}^q$ and their 3D Voronoi Cells $\{\mathcal{S}_i\}_{i=1}^q$\\
			\textbf{Output}: Query Points set $\mathcal{Q}$
			\begin{algorithmic}[1] 
				\STATE Let $\mathcal{Q}=\emptyset$.
				\FOR{Each basis kernel point $\mathbf{p}_i$}
				\STATE Determine $\epsilon$ such that $\epsilon <\texttt{radius}(\mathcal{S}_i)$.
				\STATE  $\mathcal{Q}\gets \mathcal{Q}\cup \{\mathbf{p}_i+\epsilon\hat{i},\mathbf{p}_i+\epsilon\hat{j},\mathbf{p}_i+\epsilon\hat{k}\}$ 
				\ENDFOR
				\STATE \textbf{return} $\mathcal{Q}$
			\end{algorithmic}
		\end{algorithm}
		
		\subsection{Learning LISR}
		Now, we explain our approach for finding the optimal coefficients $\boldsymbol{\alpha}=\begin{bmatrix} \alpha_1&\cdots&\alpha_{3q}\end{bmatrix}^\top$. As shown in Figure \ref{betaperturb}, the final solution is too sensitive to the coefficients, and the final reconstructed surfaces using the classical algorithms are not robust to missing parts in the input point cloud. Hence, we follow a learning-based approach to find optimal coefficients. \\
		\textbf{Predicting Kernel Centers:} Our neural network architecture consists of two sub-networks. The first network is the Basis Prediction Network, which predicts the kernel points of the Compactly Supported Radial Basis Function (CSRBF) for partial shapes. For partial shapes, the network learns to predict the optimal locations for these kernel points. For complete shapes, however, an interesting optimization comes into play: the input point cloud itself can directly serve as the kernel points for the basis, as they provide the necessary information to encode the shape's structure. The structure of the proposed basis prediction network is the simple 3D auto-encoder which is inspired by the network proposed by \cite{yuan2018pcn}. The input to this network is the partial input point cloud $\mathcal{M}=\{\mathbf{x}_1,\ldots,\mathbf{x}_n\}$. The output of this network is the set of kernel 3D points $\mathcal{P}=\{\mathbf{p}_1,\ldots,\mathbf{p}_q\}$. We use the Chamfer-L1 distance between the predicted kernel centers point cloud $\mathcal{P}$ and the input point cloud $\mathcal{M}$  as the loss function to train the kernel center prediction network defined as the below equation:
		\begin{eqnarray}
			\nonumber\ell_{1}&=&\frac{1}{|\mathcal{M}|}\sum_{\mathbf{x}\in\mathcal{M}}\underset{\mathbf{p}\in\mathcal{P}}{\min}\|\mathbf{x}-\mathbf{p}\|_1+\frac{1}{|\mathcal{P}|}\sum_{\mathbf{p}\in\mathcal{P}}\underset{\mathbf{x}\in\mathcal{M}}{\min}\|\mathbf{x}-\mathbf{p}\|_1
		\end{eqnarray}
		\textbf{Coefficient Prediction:} Once we predict the kernel points using the basis prediction network or obtained from the input point cloud for complete shapes, we pass them into the coefficient prediction network. This network is responsible for estimating the coefficients $\alpha_i$ corresponding to each basis $\phi_i$ where $i\in\{1,2,\ldots,3q\}$.\\
		By combining the kernel points and their corresponding coefficients, the linear implicit surface representation is constructed as $f(\mathbf{x}) = \sum_{i=1}^{3q} \alpha_i \phi_i(\mathbf{x})$.  To train the proposed coefficient prediction neural network, we use the following loss between the predicted SDF $f(\mathbf{x})=\sum_{i=1}^{3q}\alpha_i\phi_i(\mathbf{x})$ and the ground-truth SDF $\mathbf{s}_{\text{gt}}(\mathbf{x})$: 
		\begin{eqnarray}
			\ell_{2}&=&\sum_{\mathbf{x}\in\mathcal{Q}}\|\sum_{i=1}^{3q}\alpha_i\phi_i(\mathbf{x})-\mathbf{s}_{\text{gt}}(\mathbf{x})\|_2^2
			\label{eq11}
		\end{eqnarray}
		
		In Figure \ref{fig:fc}, we present the structure of both the kernel center prediction network and the coefficient prediction network and the overall flow of the proposed approach.
		\begin{table}[h]
			\centering
			\begin{tabular}{ccc}
				\hline\hline
				\textbf{Basis Functions} & \textbf{Rank of $\mathbf{VV}^\top$} & \textbf{ Maximum rank}  \\
				\hline
				Tri-Harmonic  & 954 & 1000\\
				Mono-Harmonic  & 997 & 1000\\
				HRBF  & 13 & 3000\\
				CSRBF (Proposed) & 3000 & 3000\\
				\hline
			\end{tabular}
			\caption{Rank of the matrix $\mathbf{VV}^\top$ for different choices of the basis functions. For the convergence of the optimization algorithm to train a coefficient prediction network, the $\mathbf{VV}^\top$ has to be a full rank matrix as shown in Theorem \ref{th_1}.  We observe that the proposed radial basis functions with compact support achieve the full rank constraint. }
			\label{singleshapetable1}
		\end{table}
		\begin{figure*}[!h]
			\centering
			\stackunder{\epsfig{figure=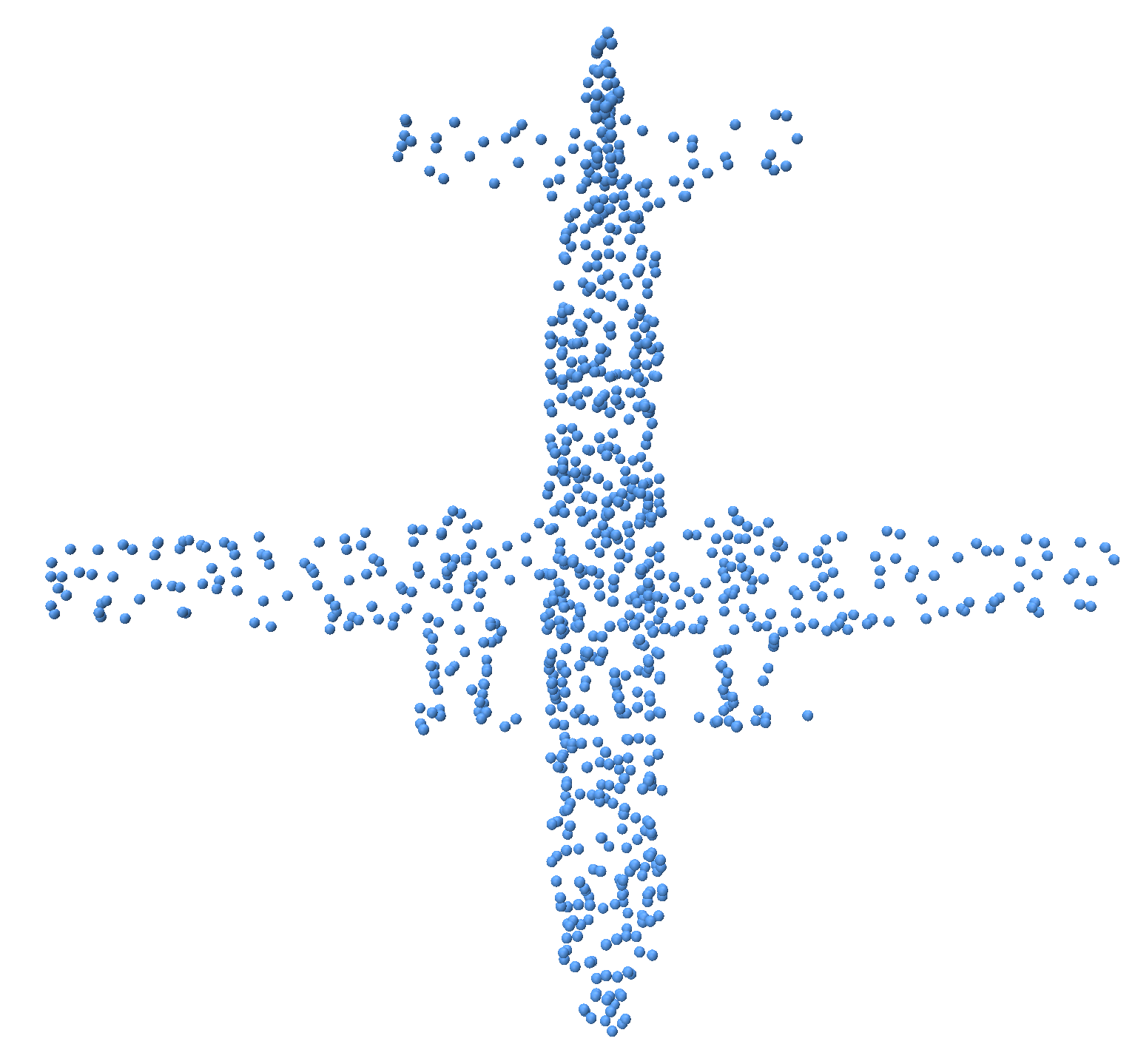,width=0.21\linewidth}}{(a)}
			\stackunder{\epsfig{figure=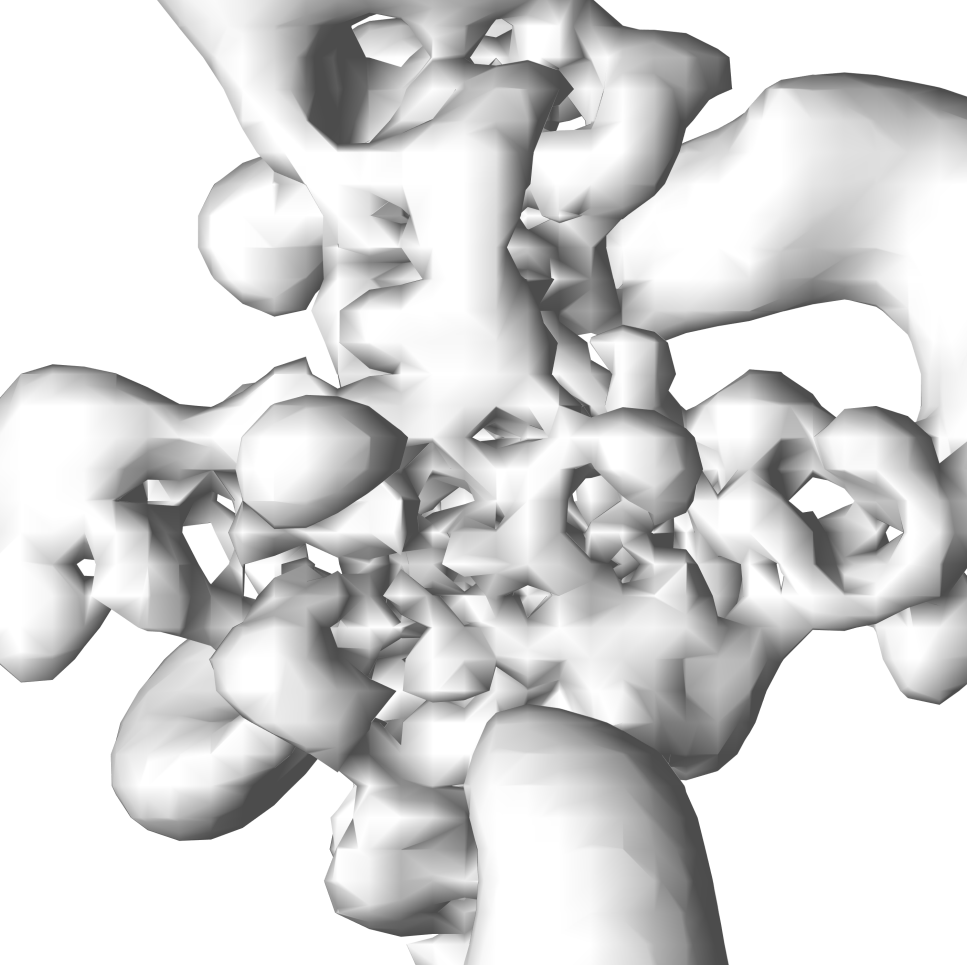,width=0.185\linewidth}}{(b)}
			\stackunder{\epsfig{figure=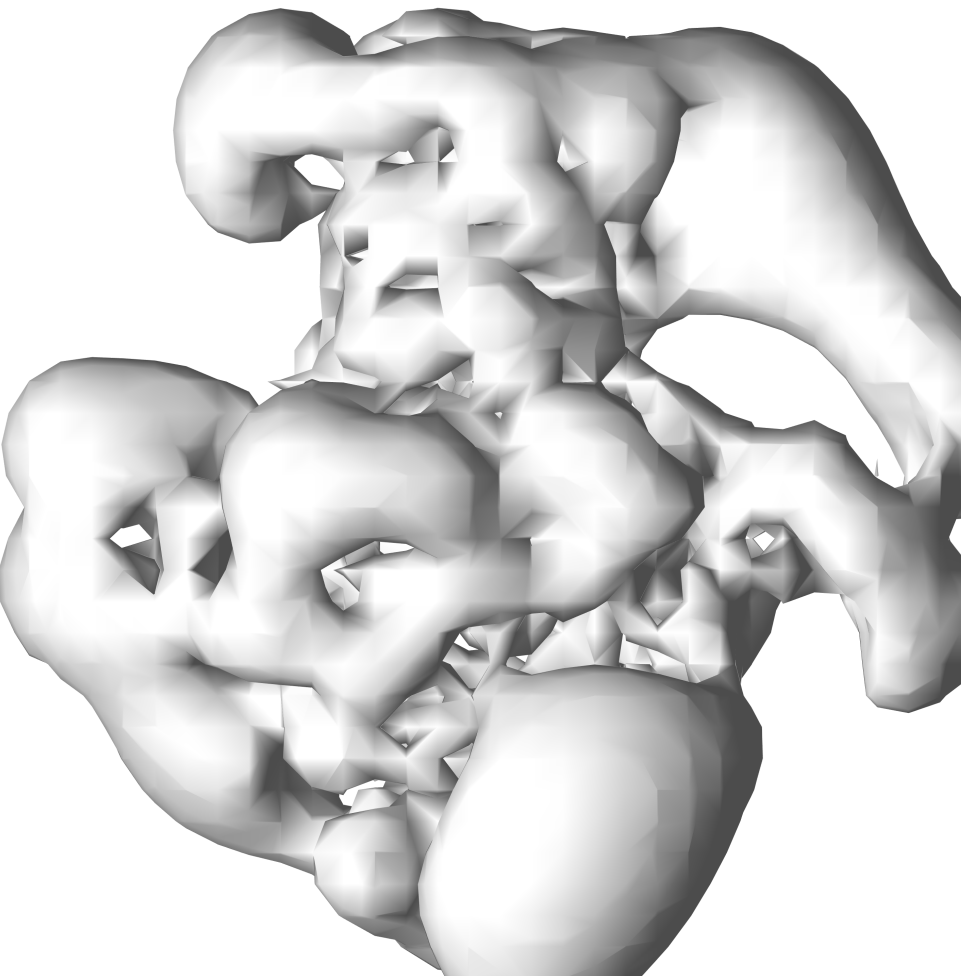,width=0.185\linewidth}}{(c)}
			\stackunder{\epsfig{figure=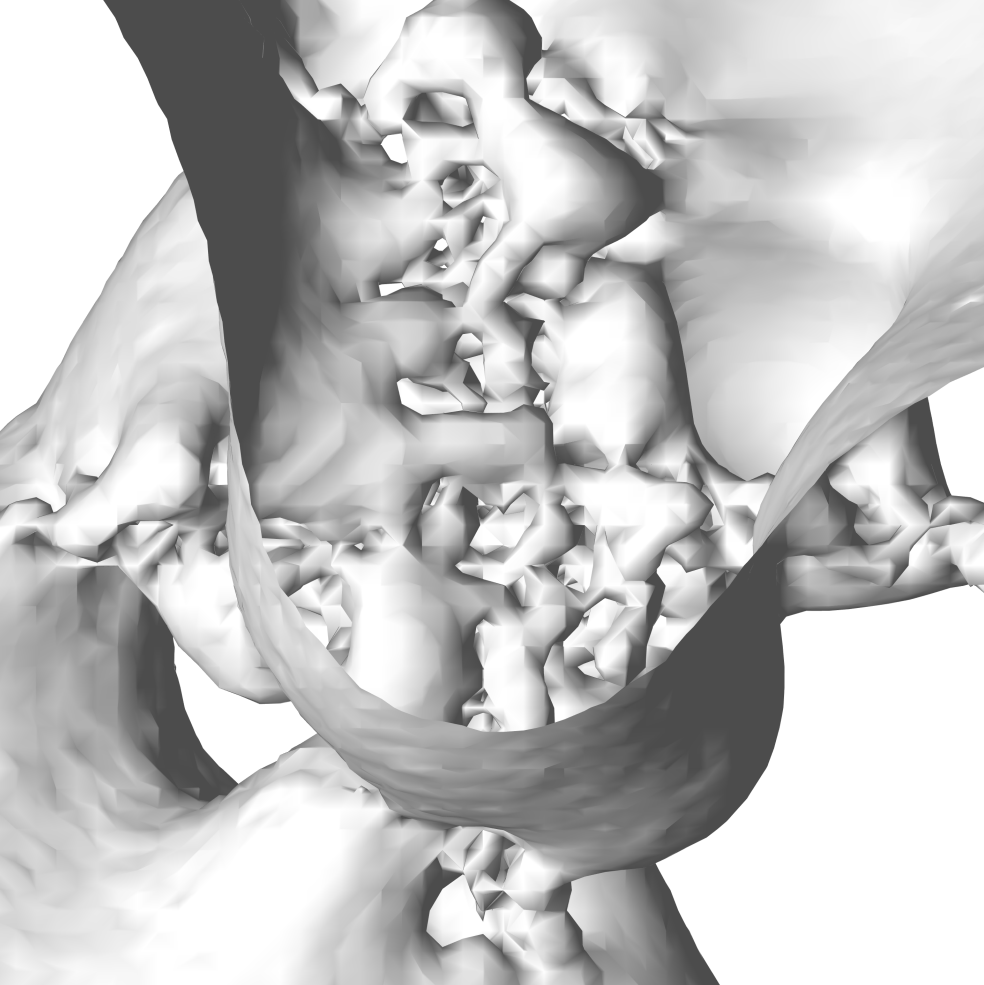,width=0.185\linewidth}}{(d) }
			\stackunder{\epsfig{figure=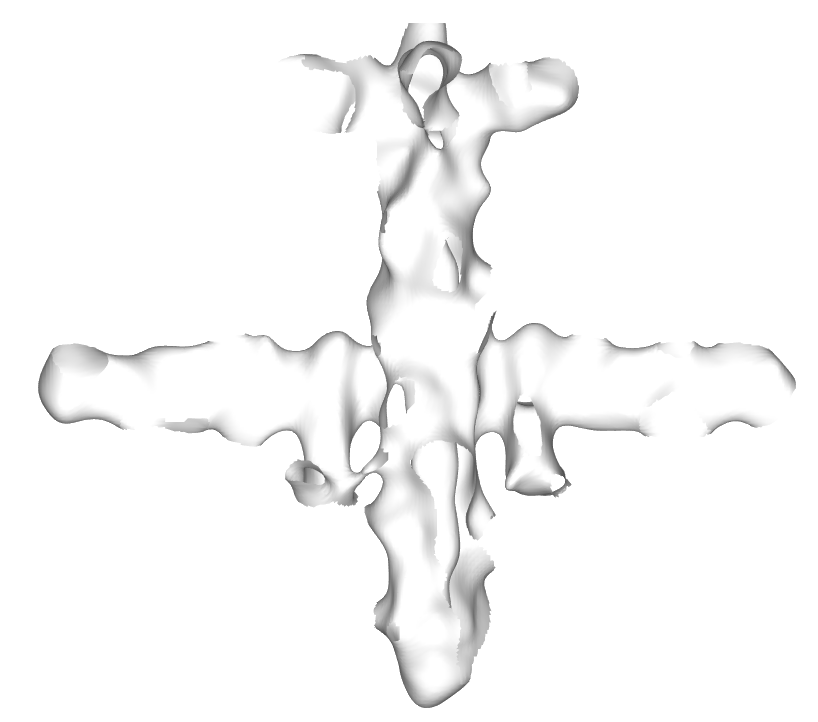,width=0.21\linewidth}}{(e)}
			\caption{Single Shape optimization solution on the Point Cloud (shown (a)) for (b) Tri-harmonic radial basis, (c) Mono-harmonic radial basis, (d) Hermite-Radial basis with uniform query point selection strategy in the learning-based framework and  (e) LISR with query point selection strategy defined in Algorithm \ref{alg:algorithm2} over $[-1,1]^3$ optimization domain and single shape overfitting with only 1000 kernel points. LISR can capture the global shape properties.}
			\label{singleshape}
		\end{figure*}
		
		\section{Experiments and Results}
		In this section, we show the improvement in learning ability with our choice of basis and query points over other basis functions such as Tri-harmonic, Mono-harmonic, and  Hermite-RBFs. We evaluate the ability of our model to generalize over a distribution of shapes by training it on various classes of shapes in the ShapeNet dataset \cite{chang2015shapenet}. We also evaluate the reconstruction of the shapes from partial point clouds by coupling the model with the Basis Prediction Network.
		\subsection{Single Shape Learning}
		We demonstrate the improved shape representation and learning of our approach by comparing our choice of basis and query points against various choices of basis functions with global support such as Tri-harmonic RBFs, mono-harmonic, and HRBF\cite{liu2016closed} with uniformly sampled query points from the optimization domain $[-1,1]^3$. \\
		We formulate the shape representation as given in Equation \eqref{eq8} and use ADAM optimizer to compute the optimal value for $\boldsymbol{\alpha}$ for the loss function given in Equation \eqref{eq11} for over-fitting on a single shape. We compute the rank of the matrix $\mathbf{VV}^\top$ matrix for each choice of basis and their corresponding query points as given in Table \ref{singleshapetable1}. We also perform qualitative analysis by comparing the shape extracted after zero-level iso-surface extraction given in Figure \ref{singleshape}. For the convergence of the optimization algorithm to train a coefficient prediction network, the $\mathbf{VV}^\top$ has to be a full rank matrix as shown in Theorem \ref{th_1}.  We observe that the proposed radial basis functions with compact support achieve the full rank constraint. We have used 1000 kernel centers for this experiment. We observe that, in the learning-based framework, the proposed LIRS can learn the approximate shape of the actual object even with 1000 kernel centers on over-fitting a large network with a single step. Whereas, with other basis functions, it is not able to capture any aspect of the object surface.
		\begin{table*}[h]
			\centering
			\begin{tabular}{ccccccccccc}
				\hline\hline
				& \multicolumn{2}{c}{\textbf{OccNet}} & \multicolumn{2}{c}{\textbf{ConvNet}}&\multicolumn{2}{c}{\textbf{IF-Net}} &\multicolumn{2}{c}{\textbf{3DILG}}&\multicolumn{2}{c}{\textbf{LIRS}}\\
				& \multicolumn{2}{c}{\cite{mescheder2019occupancy}} &  \multicolumn{2}{c}{\cite{peng2020convolutional}}& \multicolumn{2}{c}{(Chibane et al. 2020)}&\multicolumn{2}{c}{(Zhang et al. 2023)}&\multicolumn{2}{c}{ (Proposed)}\\ \cline{2-11}
				Class&CD$\downarrow$&F-Score$\uparrow$&CD$\downarrow$&F-Score$\uparrow$&CD$\downarrow$&F-Score$\uparrow$&CD$\downarrow$&F-Score$\uparrow$&CD$\downarrow$&F-Score$\uparrow$\\
				\cline{1-11}
				Chair&0.058&0.890&0.044&0.934&0.031&0.990&0.029&0.992&\textbf{0.027}&\textbf{0.993}\\
				Aeroplane&0.037&0.948&0.028&0.982&0.020&0.994&\textbf{0.019}&0.993&0.022&\textbf{0.997}\\
				Lamp&0.090&0.820&0.050&0.945&0.038&0.970&0.036&0.971&\textbf{0.028}&\textbf{0.980}\\
				Sofa&0.051&0.918&0.042&0.967&0.032&0.988&0.030&0.986&\textbf{0.027}&\textbf{0.991}\\
				Table&0.041&0.961&0.036&0.982&0.029&0.998&\textbf{0.026}&\textbf{0.999}&0.033&0.974\\\hline
				Mean&0.046&0.907&0.040&0.962&0.030&\textbf{0.988}&0.028&\textbf{0.988}&\textbf{0.027}&0.987\\\hline
			\end{tabular}
			\caption{Comparison of the performance of the proposed approach (LISR) with the performance of the sate-of-the-art approaches OccNet \cite{mescheder2019occupancy}, ConvNet \cite{peng2020convolutional}, IF-Net \cite{chibane2020implicit}, and 3DILG \cite{zhang20223dilg}. We have used the standard evaluation metrics for this task Chamfer Distance (CD) and the F-Score and five classes (Chair, Aeroplane, Lamp, Sofa, and Table) from the ShapeNet dataset \cite{chang2015shapenet}.}    
			\label{singleshapetable}
		\end{table*}
		\subsection{3D Surface Reconstruction Evaluation} To check the performance of the proposed approach for the task of 3D implicit surface fitting, we compare the performance of the proposed approach (LISR) with the performance of the sate-of-the-art approaches OccNet \cite{mescheder2019occupancy}, ConvNet \cite{peng2020convolutional}, IF-Net \cite{chibane2020implicit}, and 3DILG \cite{zhang20223dilg}. We use the standard evaluation metrics for this task such as Chamfer-L1 distance (CD) and the F-Score \cite{zhang20223dilg,tatarchenko2019single}. We use five classes (Chair, Aeroplane, Lamp, Sofa, and Table) from the ShapeNet dataset \cite{chang2015shapenet} to train our network and compare the performance with that of the state-of-the-art approaches. We follow the (23023,575,1152) train-val-test split on ShapeNet. Following the evaluation protocol of  \cite{zhang20223dilg}, we include two metrics, the Chamfer-L1 distance, and F-Score. In our experiment, we set the threshold equal to $0.02$. In Table \ref{singleshapetable}, we present the Chamfer Distance and the F-Score for the proposed approach and the state-of-the-art approaches. We observe that, on average, the proposed approach achieves a marginally better Chamfer distance and marginally lesser F-Score than that of the 3DILG \cite{zhang20223dilg} and IF-Net \cite{chibane2020implicit}. In Figure \ref{res3}, we show the results of the proposed approach for three objects along with the ground-truth 3D surfaces. We observe that the reconstructed implicit surface captures topology information but in the case of the chair, it could not capture one of the spikes in the support of the chair. 
		
		\begin{figure}[!h]
			\centering
			\stackunder{\epsfig{figure=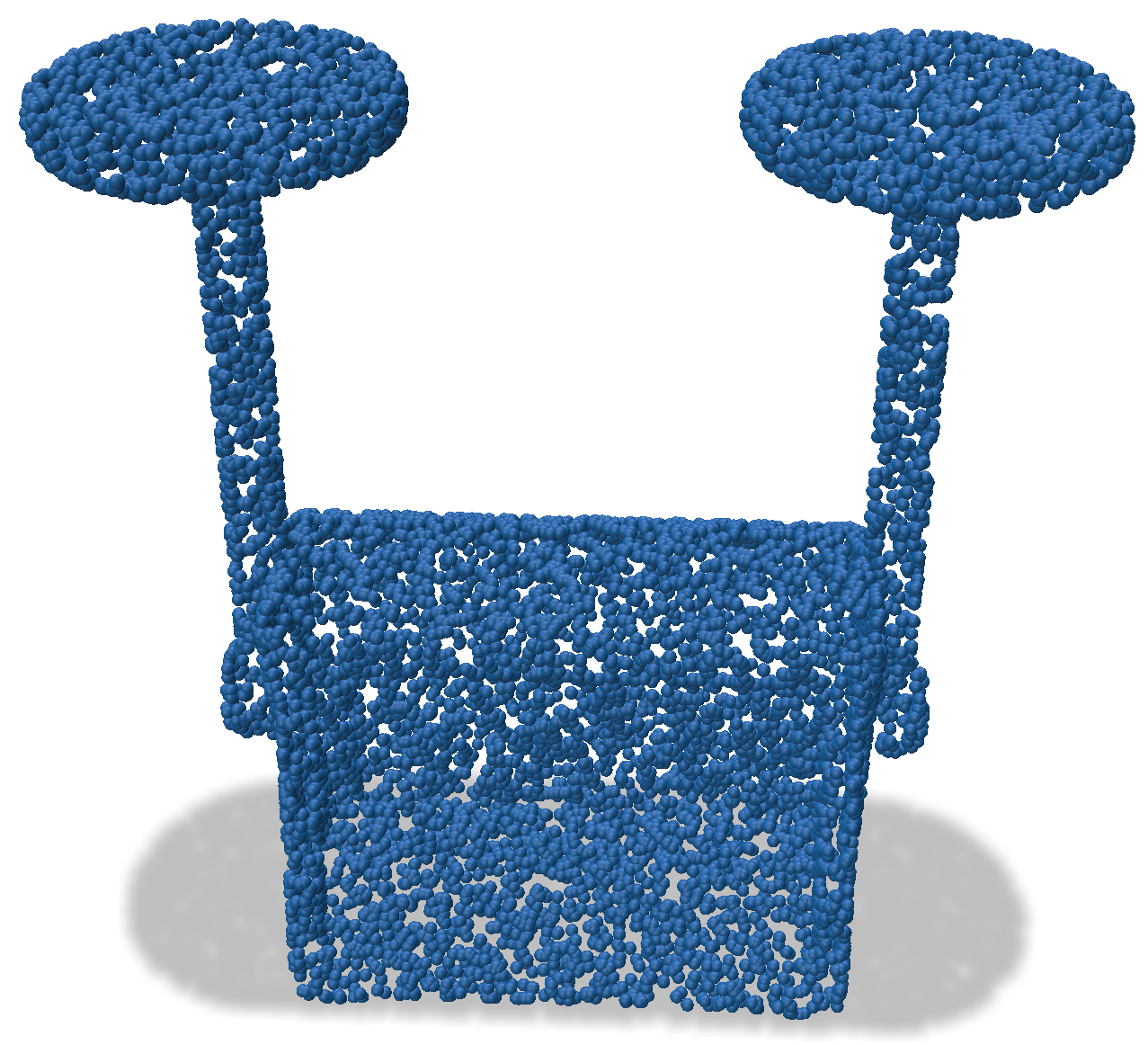,width=0.31\linewidth}}{}
			\stackunder{\epsfig{figure=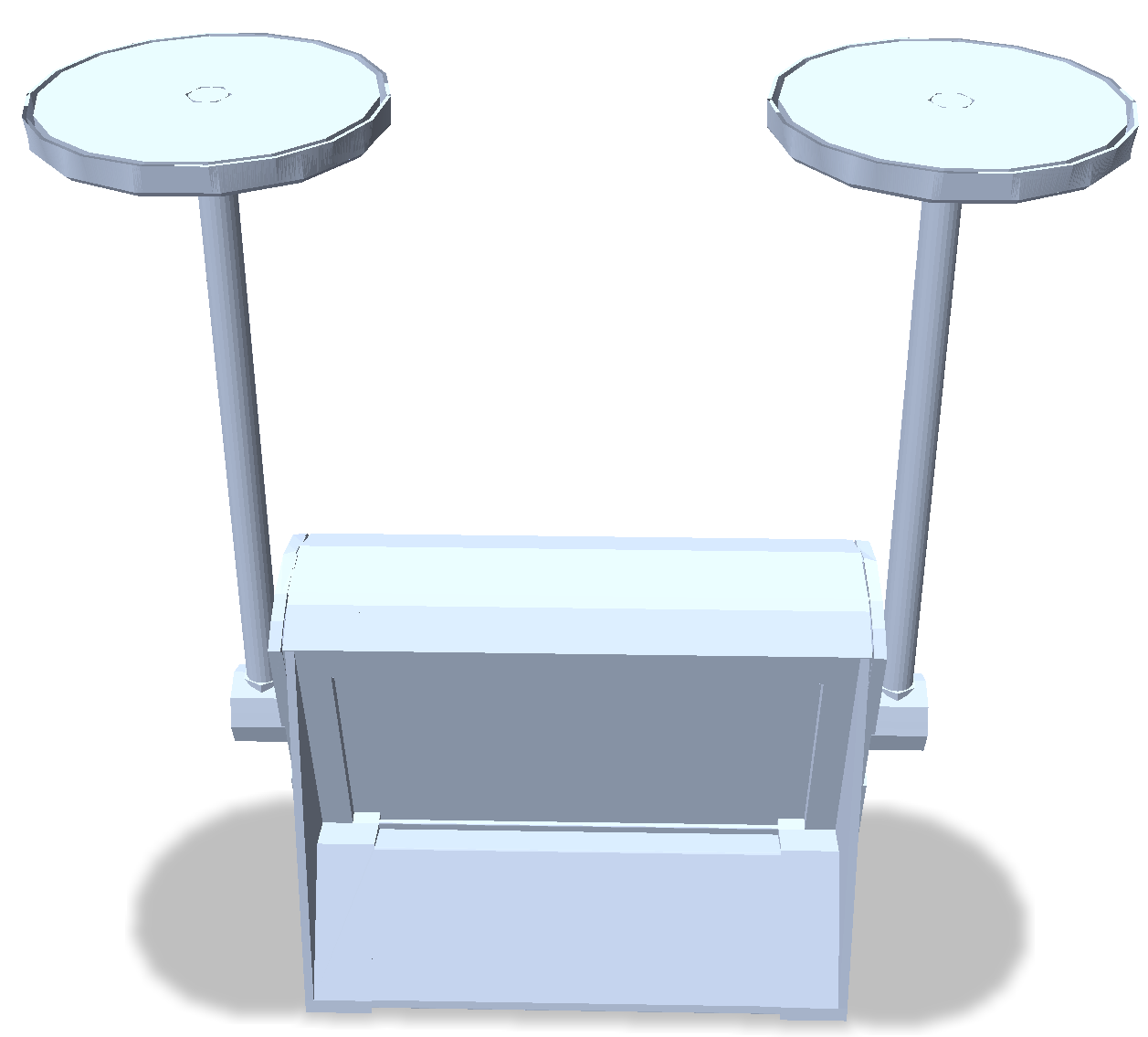,width=0.31\linewidth}}{}
			\stackunder{\epsfig{figure=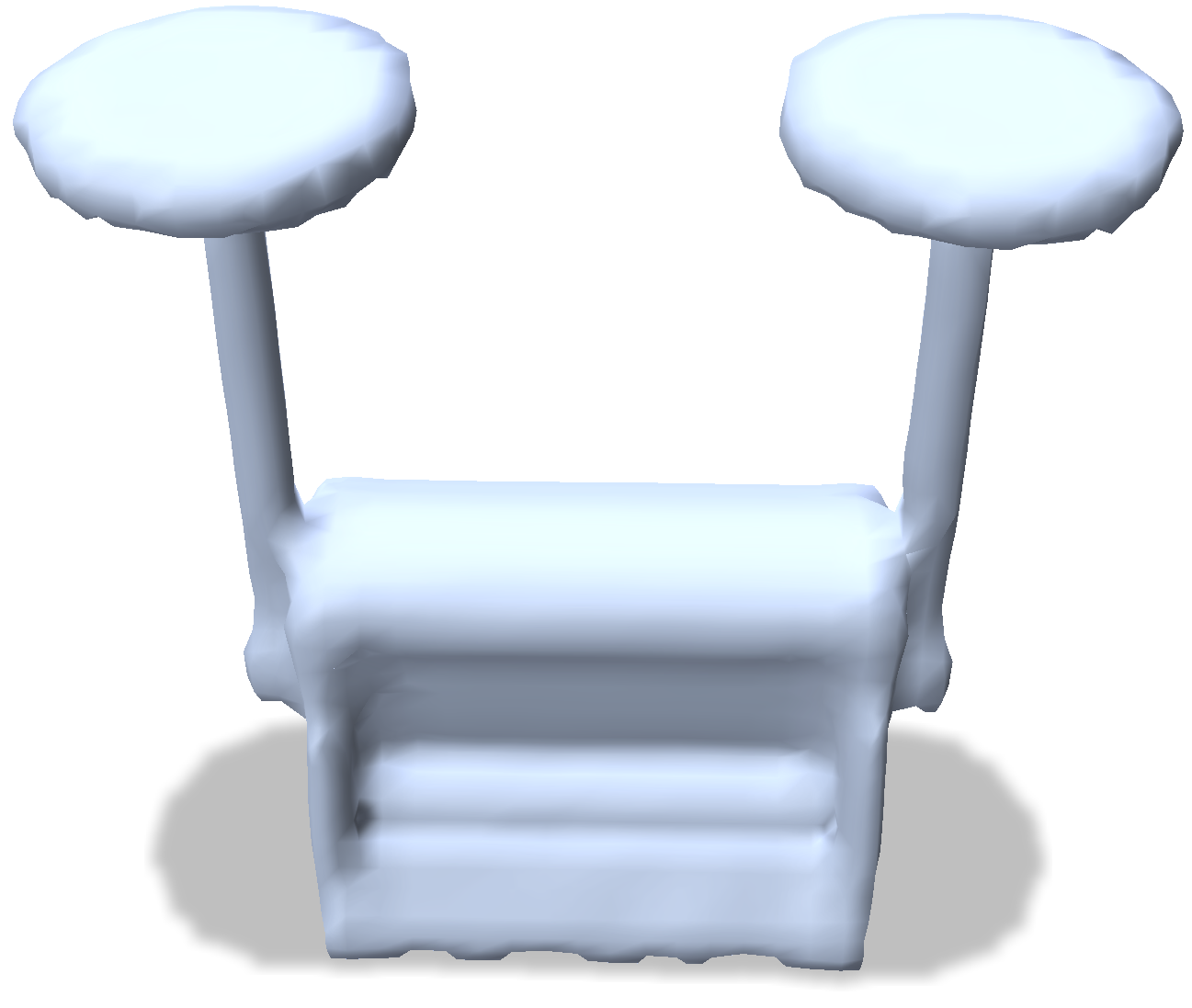,width=0.315\linewidth}}{}
			\stackunder{\epsfig{figure=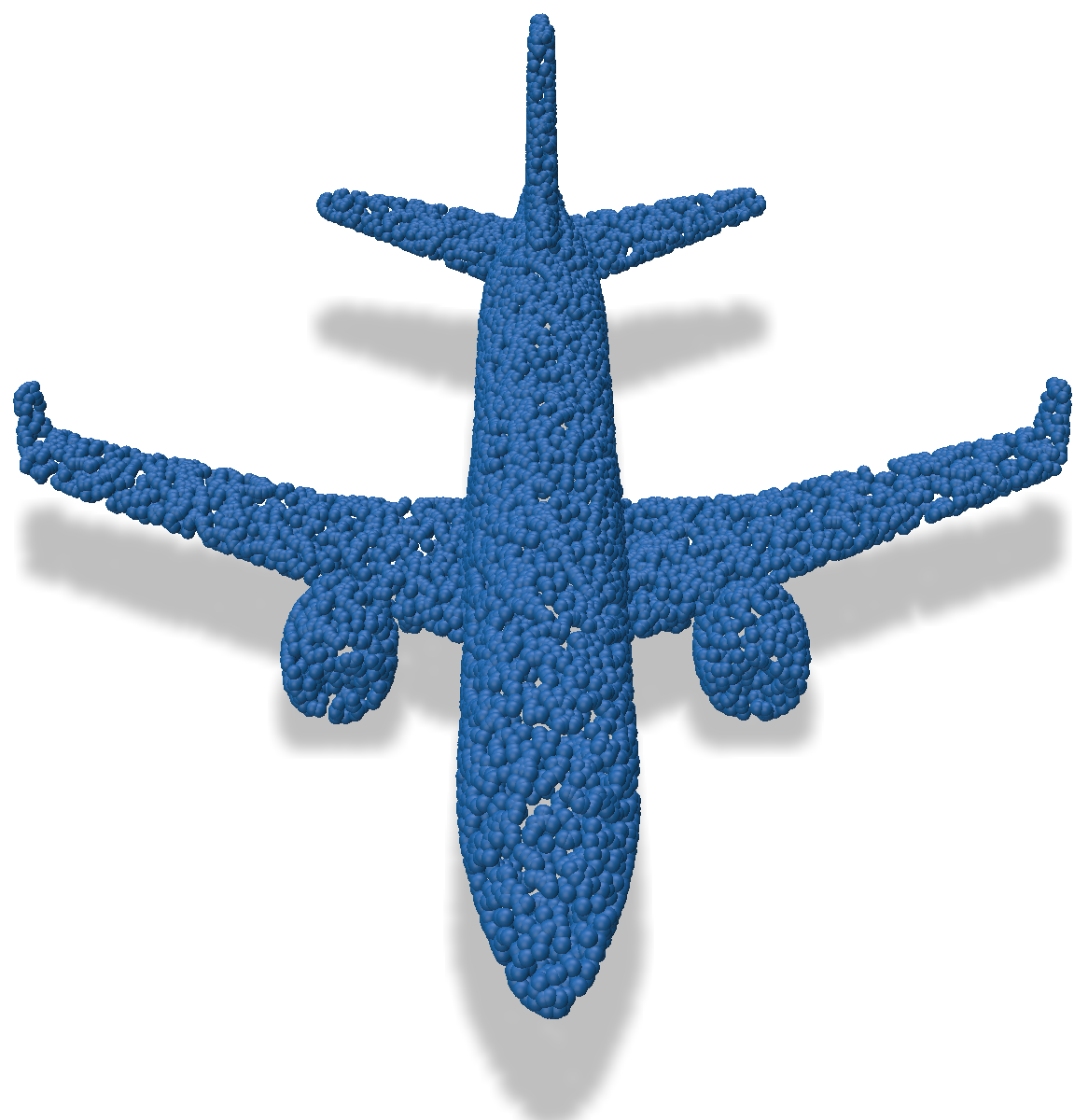,width=0.325\linewidth}}{}
			\stackunder{\epsfig{figure=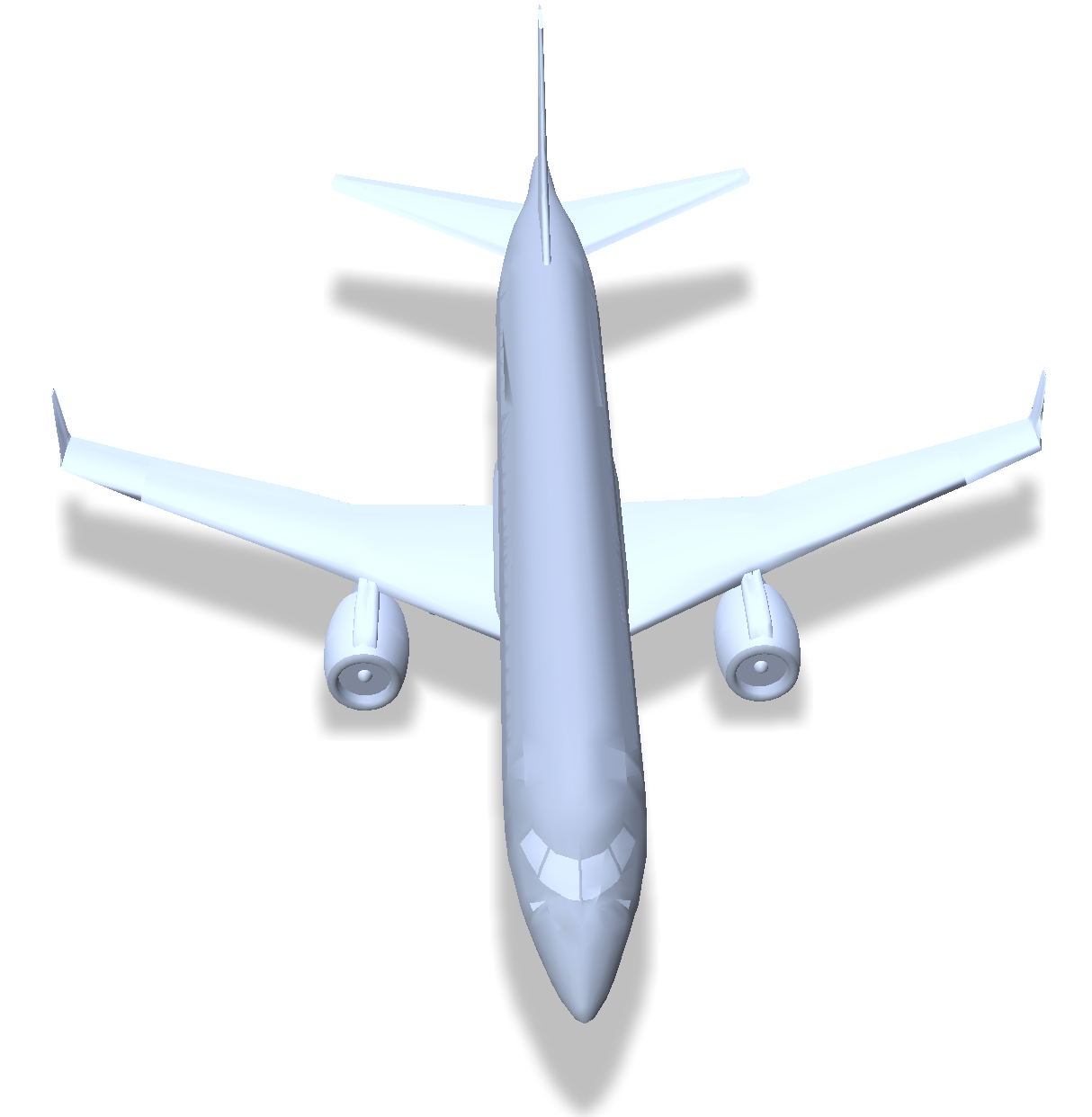,width=0.325\linewidth}}{}
			\stackunder{\epsfig{figure=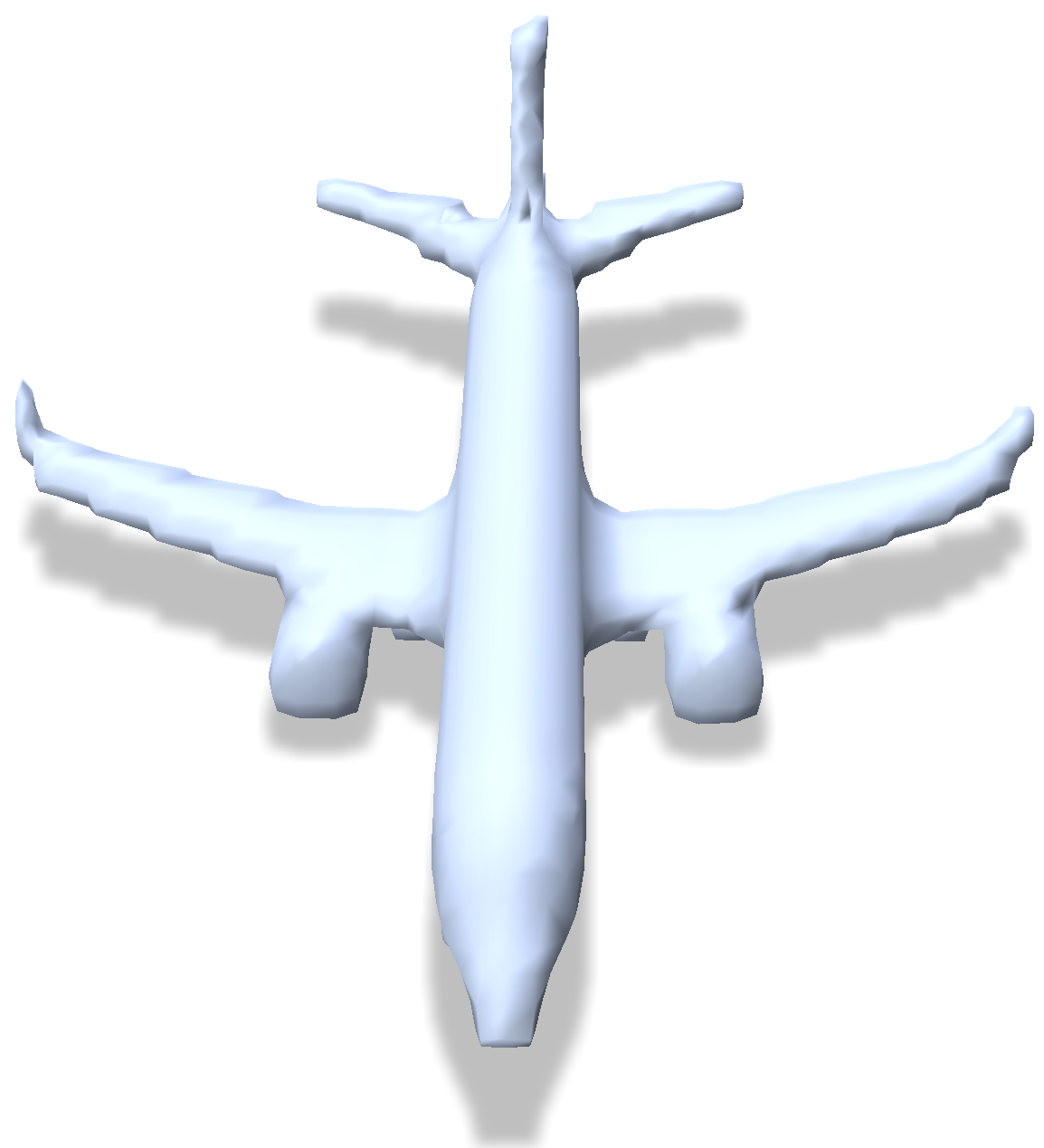,width=0.325\linewidth}}{}
			\stackunder{\epsfig{figure=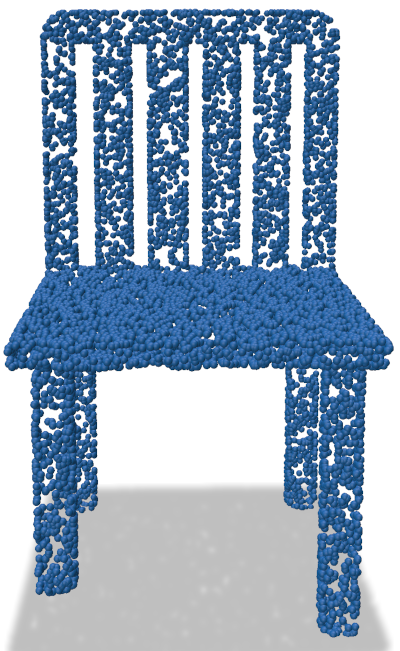,width=0.315\linewidth}}{(a) Input}
			\stackunder{\epsfig{figure=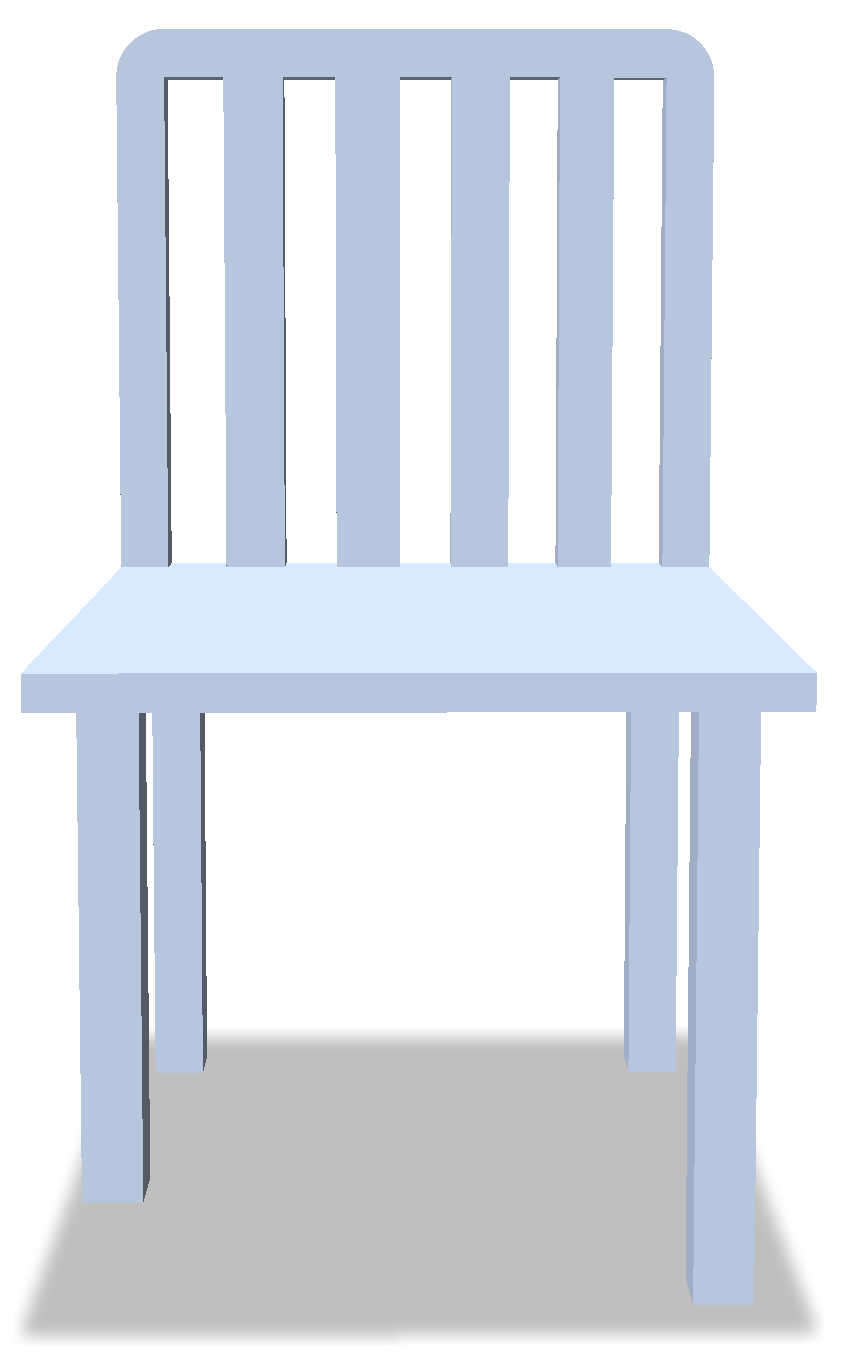,width=0.32\linewidth}}{(b) Ground-truth}
			\stackunder{\epsfig{figure=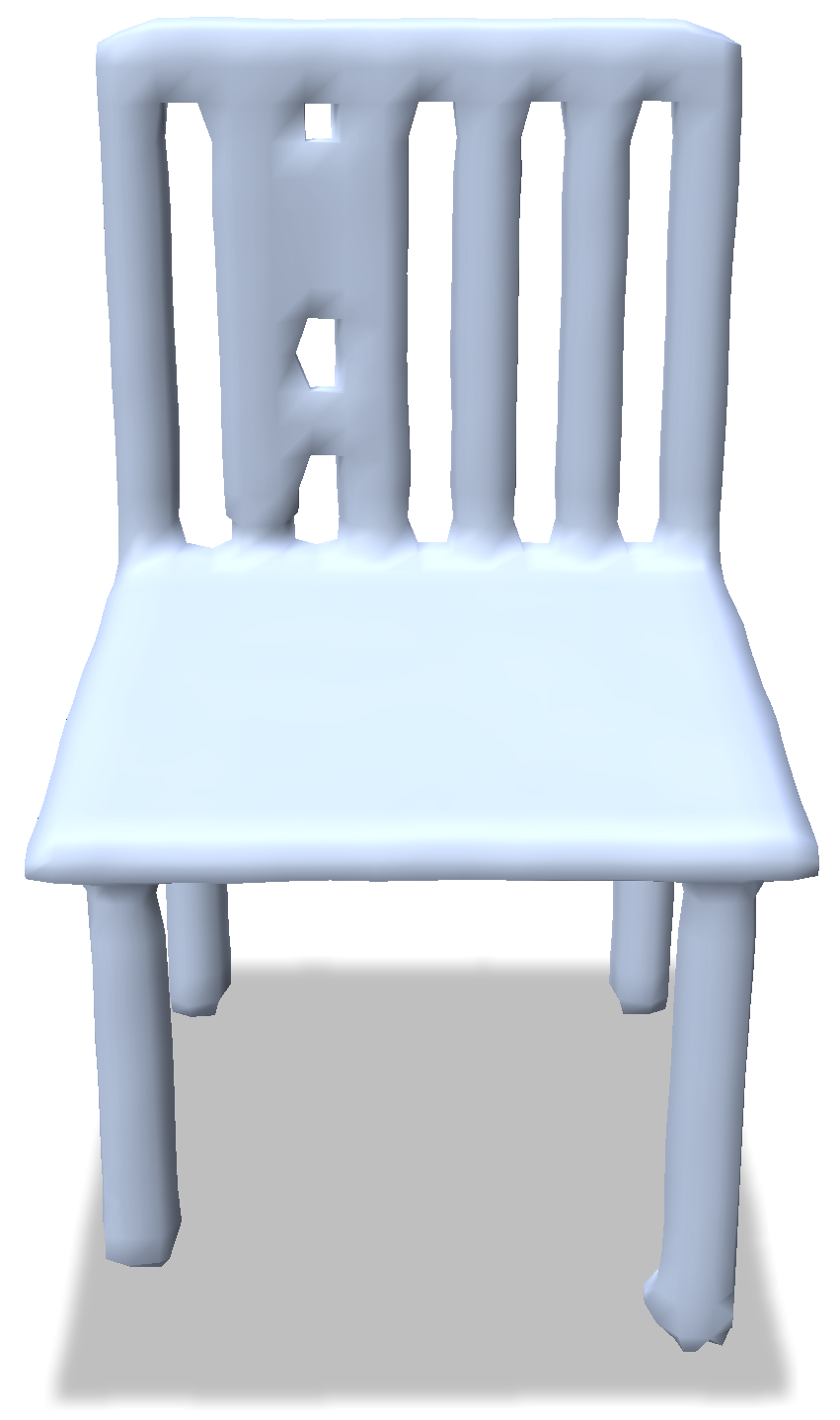,width=0.305\linewidth}}{(c) Predicted}
			\caption{Visual results of 3D implicit surface reconstruction on three input point clouds.}
			\label{res3}
		\end{figure}

		\textbf{Network and Training Detail, Run-time }: Our architecture consists of two sub-networks: Basis Prediction Network (BPN) and Coefficient Prediction Network (CPN). It is essential for the CPN to be permutation invariant since there exists a correspondence between the basis and its coefficient. We ensure this through the use of shared MLPs. CPN takes kernel points as input and passes them through a shared-MLP (6 layers with 3, 64,129, 256, 512, 1024 neurons, resp.) with a max-pooling after the last layer to generate 1024 global latent which is then concatenated with each of the kernel points and then passed through two shared MLPs (5 layers with  1024+3, 512, 512, 512, 512 neurons, resp.) and (5 layers with 1027+512, 512, 512, 512, 3 neurons, resp.) with a skip connection between the two to generate the beta values. The model contains 6670122 trainable parameters. In essence, the basis kernel points provide a fixed foundation, and the coefficients control the influence of each basis function. The Voronoi cells of kernel points are used as their local support for the basis functions while training. We use ADAM with learning rate = 1e-8, weight\_decay = 1e-7 along with OneCycleLR scheduler with max\_lr=1e-3. We train the proposed network on a 48GB RAM
		GPU on a Linux machine and have used PyTorch (Paszke
		et al. 2019). The average time for a forward pass on test point clouds having 10,000 points is around 0.0139 Seconds. 
		\subsubsection{3D Shape Completion}
		To show the effectiveness of the proposed approach, we trained the proposed architecture for the 3D shape completion task on the MVP dataset \cite{pan2021robust,pan2021variational} that consists of eight object classes. To evaluate our approach, we find the Chamfer-L1 distance between the predicted 3D surface and the ground-truth complete point clouds. The average Chamfer distance for the proposed approach for all 8 classes is  $0.039$.
		\section{Conclusion}
		In this work, we have addressed the problem of reconstructing the implicit 3D surface of an object from its partial 3D point cloud. We have proposed a radial basis functions-based linear representation of the underlying implicit surface. We further showed that the existing radial basis-based representation for implicit surface fail in a learning-based framework as the SDF predicted by these methods does not satisfy the full rank matrix criterion on the set of query points required for finding the mean squared loss between the predicted and the ground-truth SDFs. We have proposed an RBF-based approach where we learn linear radial basis functions with compact local supports. With this representation, along with the proposed query point selection, we were able to learn the coefficients of the proposed RBF-based representation. The proposed parametrization of the implicit surface is linear in the coefficients of RBF and has a compact surface representation. We achieved the best performance with respect to the Chamfer distance metric and comparable performance with the F-score metric.  
		\\\textbf{Limitations and Future Scope:} Our approach mostly supports single object-based tasks as it was trained on the ShapeNet dataset, and also it requires supervision from the ground-truth signed distance field. In future work, we would like to extend the proposed framework for scene-level tasks and solve it in semi-supervised or unsupervised approaches.
		\section*{Acknowledgments}
		This research was supported by the start-up research grant (SRG), SERB, Government of India and TIH iHub Drishti, IIT Jodhpur.

		\bibliography{aaai24}

\end{document}